\documentclass{ecai} 
\usepackage{latexsym}
\usepackage{amssymb}
\usepackage{amsmath}
\usepackage{amsthm}
\usepackage{booktabs}
\usepackage{enumitem}
\usepackage{graphicx}
\usepackage{orcidlink}
\usepackage{color}
\usepackage{multirow}
\usepackage{makecell}
\usepackage{hyperref}
\usepackage{amstext, mathtools}
\usepackage{algorithm}
\usepackage{algpseudocode}
\usepackage{graphicx}
\usepackage{soul}
\usepackage{todonotes}

\DeclareMathOperator*{\avg}{avg}

\newcommand{\mv}{\mathbf{v}}
\newcommand{\X}{\mathcal{X}}
\newcommand{\F}{\mathbb{F}}
\DeclareMathOperator*{\argmax}{argmax}

\newtheorem{theorem}{Theorem}
\newtheorem{lemma}{Lemma}

\newtheorem{proposition}[theorem]{Proposition}

\newtheorem{definition}{Definition}

\newtheorem{example}{Example}
\newtheorem*{excont}{Example \continuation}
\newcommand{\continuation}{??}
\newenvironment{continueexample}[1]
 {\renewcommand{\continuation}{\ref{#1}}\excont[Cont]}
 {\endexcont}
\newcommand{\BibTeX}{B\kern-.05em{\sc i\kern-.025em b}\kern-.08em\TeX}

\begin{document}

%%%%%%%%%%%%%%%%%%%%%%%%%%%%%%%%%%%%%%%%%%%%%%%%%%%%%%%%%%%%%%%%%%%%%%%%

\begin{frontmatter}

%%% Use this command to specify your submission number.
%%% In doubleblind mode, it will be printed on the first page.

\paperid{248} 

%%% Use this command to specify the title of your paper.

\title{Backward explanations via redefinition of predicates}

%%% Use this combinations of commands to specify all authors of your 
%%% paper. Use \fnms{} and \snm{} to indicate everyone's first names 
%%% and surname. This will help the publisher with indexing the 
%%% proceedings. Please use a reasonable approximation in case your 
%%% name does not neatly split into "first names" and "surname".
%%% Specifying your ORCID digital identifier is optional. 
%%% Use the \thanks{} command to indicate one or more corresponding 
%%% authors and their email address(es). If so desired, you can specify
%%% author contributions using the \footnote{} command.

\author[A]{\fnms{Léo}~\snm{Saulières}\orcidlink{0000-0002-4800-9181}\thanks{Corresponding author email: leo.saulieres@irit.fr.}}
\author[A]{\fnms{Martin C.}~\snm{Cooper}\orcidlink{0000-0003-4853-053X}}
\author[B]{\fnms{Florence}~\snm{Dupin de Saint-Cyr}\orcidlink{0000-0001-7891-9920}} 

\address[A]{IRIT, University of Toulouse III, CNRS, Toulouse, France}
\address[B]{Lab-STICC\_COMMEDIA, CNRS, Brest, France.}

%%% Use this environment to include an abstract of your paper.

\begin{abstract}
History eXplanation based on Predicates (HXP), studies the behavior of a Reinforcement Learning (RL) agent in a sequence of agent's interactions with the environment (a history), through the prism of an arbitrary predicate~\cite{saulieres:hal-04170188}. To this end, an action importance score is computed for each action in the history. The explanation consists in displaying the most important actions to the user.
As the calculation of an action's importance is \#W[1]-hard, it is necessary for long histories to approximate the scores, at the expense of their quality. We therefore propose a new HXP method, called Backward-HXP, to provide explanations for these histories without having to approximate scores. %Experiments confirm the usefulness of Backward-HXP.
Experiments show the ability of B-HXP to summarise long histories. 
\end{abstract}

\end{frontmatter}

%%%%%%%%%%%%%%%%%%%%%%%%%%%%%%%%%%%%%%%%%%%%%%%%%%%%%%%%%%%%%%%%%%%%%%%%
\section{Introduction}
%%%%%%%%%%%%%%%%%%%%%%%%%%%%%%%%%%%%%%%%%%%%%%%%%%%%%%%%%%%%%%%%%%%%%%%%

% Start / introduction
Nowadays, Artificial Intelligence (AI) models are used in a wide range of tasks in different fields, such as medicine, agriculture and education~\cite{hamet2017artificial, eli2019applications, DBLP:journals/access/ChenCL20a}. 
Most of these models cannot be explained or interpreted without specific tools, mainly due to the use of neural networks which are effectively black-box functions. Numerous institutions~\cite{noauthor_blueprint_2022, EU_2021} and researchers~\cite{DBLP:journals/cacm/Darwiche18, DBLP:journals/cacm/Lipton18} %DBLP:journals/cacm/WeldB19
have emphasized the importance of providing comprehensible models to end users. 
This is why the eXplainable AI (XAI) research field, which consists in providing methods to explain AI behavior, is flourishing. In this context, we propose a method for explaining AI models that have learned using Reinforcement Learning (RL).

% RL
In RL, the agent learns by trial and error to perform a task in an environment.
At each time step, the agent chooses an action from a state, arrives in a new state and receives a reward. The dynamics of the environment are defined by the non-deterministic transition function and the reward function. The agent learns a policy $\pi$ to maximize its reward; this policy assigns an action to each state (defining a deterministic policy). Our eXplainable Reinforcement Learning (XRL) method is restricted to the explanation of deterministic policies.

% XRL small overview
Various works focus on explaining RL agents using a notion of importance. 
To provide a visual summary of the agent's policy, Amir and Amir~\cite{DBLP:conf/atal/AmirA18} select a set of interactions of the agent with the environment (sequences)
using ``state importance''~\cite{ClouseStateImp}.
From a set of sequences, Sequeira et Gervasio propose to learn a set of information, to deduce interesting elements to show the user in the form of a visual summary~\cite{DBLP:journals/ai/SequeiraG20}. 
Using a self-explainable model, Guo et al. determine the critical time-steps of a sequence for obtaining the agent's final reward~\cite{DBLP:conf/nips/GuoWKX21}.

% Contribution
To explain an RL agent, explanation must capture concepts of RL~\cite{DBLP:journals/corr/abs-2202-08434}. To this end, the HXP method~\cite{saulieres:hal-04170188}, consists of studying a history of agent interactions with the environment through the prism of a certain predicate, a history being a sequence of pairs (state, action). A predicate $d$ 
is any boolean function of states which is true in the final state of the history.%can represent any partial description of states. 
This XRL method answers the question: \emph{``Which actions were important to ensure that $d$ was achieved, given the agent’s policy $\pi$?"}. This paper follows this paradigm of explanation of a history with respect to a predicate, by proposing a new way of defining the important actions of a history, called Backward-HXP (B-HXP). 
In particular, B-HXP was investigated because of the limits of (forward) HXP in explaining long histories due to the $\#$W[1]-hardness of HXP~\cite{saulieres:hal-04170188}.

% Plan
The paper is structured as follows. The theoretical principle of HXP is outlined in Section 2, before defining B-HXP in Section 3. Section 4 presents the experimental results carried out on 3 problems. Section 5 presents related works and Section 6 concludes.

    % FL B-HXP
    \begin{figure*}
        \centering
        \includegraphics[scale=0.033]{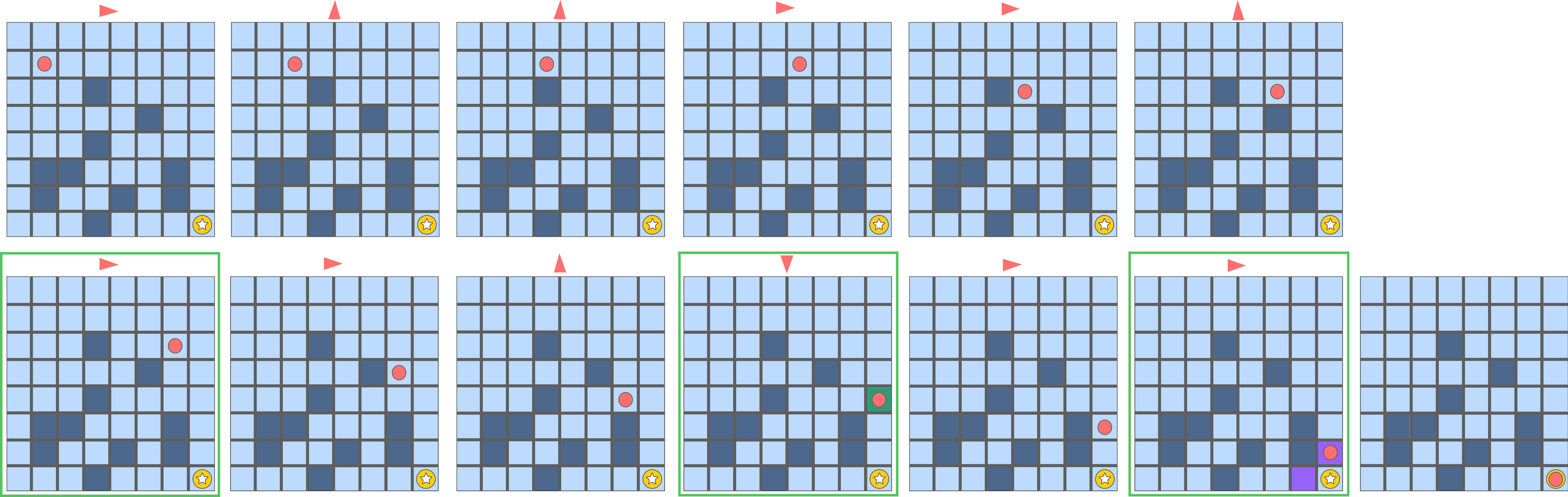}
        \caption{An example of history with 13 states $s_0,\ldots, s_{12}$ for the FL problem on which a B-HXP is computed for the \emph{win} predicate. The agent is symbolized by a red dot, the dark blue cells are holes and the destination cell is marked by a star. 
        Actions identified as important are highlighted by a green frame. \\ }
        \label{fig:FL_example}
    \end{figure*}

%%%%%%%%%%%%%%%%%%%%%%%%%%%%%%%%%%%%%%%%%%%%%%%%%%%%%%%%%%%%%%%%%%%%%%%%
\section{History eXplanation via Predicates (HXP)}
%%%%%%%%%%%%%%%%%%%%%%%%%%%%%%%%%%%%%%%%%%%%%%%%%%%%%%%%%%%%%%%%%%%%%%%%

% RL notation
An RL problem is modeled using a Markov Decision Process~\cite{sutton2018reinforcement}, which is a tuple $\langle \mathcal{S}, \mathcal{A}, R,p\rangle$. 
$\mathcal{S}$ represents the state space and $\mathcal{A}$ the action space. $A(s)$ denotes the set of available actions that can be performed from $s$. $R: \mathcal{S} \times \mathcal{A} \rightarrow \mathbb{R}$ and $p : \mathcal{S} \times \mathcal{A} \rightarrow Pr(\mathcal{S})$ are respectively the reward function and the transition function of the environment. 
$p(s'|s,a)$ represents the probability of reaching state $s'$, having performed action $a$ from state $s$. $\pi: \mathcal{S} \rightarrow \mathcal{A}$ denotes a deterministic policy that maps an action $a$ to each state $s$; thus, $\pi(s)$ is the action performed by the agent in state $s$. 
Due to the Markovian nature of the process, transitions at different instants are independent and hence the probabilities given by the transition function $p$ can be multiplied when calculating the probability of a scenario
(sequence of states).
%Starting by doing an action $a$ from a state $s$, the probability of a state $s'$ is the product of the probabilities along the current path reaching $s'$ according to $\pi$ and $p$. 
In the following, we use the function $next$, based on the policy $\pi$ and the transition function described by $p$, to compute the next possible states (associated with their probabilities) given a set $S$ of (state, probability) pairs: $next_{\pi,p}(S) = \{(s', pr \times p(s'|s, a))$ : $(s, pr) \in S$, $a = \pi(s)$ and $p(s'|s, a) \neq 0\}$.
In order to compute the set of final states reachable at horizon $k$ from a set of states $S$, using the agent's policy $\pi$ and the transition function $p$, the function \emph{succ$^k_{\pi,p}$} is defined recursively by $succ_{\pi,p}^0(S) = S$ and $succ_{\pi,p}^{n+1}(S) = next_{\pi,p}(succ_{\pi,p}^n(S))$.

% HXP
HXPs \cite{saulieres:hal-04170188} provide to the user important actions for the respect of a predicate $d$, given an agent's policy $\pi$, by computing an importance score for each action in the history. The language used for the predicate is based on the features that characterize a state. Each feature $f_i$ has a range of values defined by a domain $D_i$, the set of all features is denoted $\mathcal{F}=\{f_1, ..., f_n\}$. The feature space is therefore $\mathbb{F} = D_1 \times ... \times D_n$. The state space $\mathcal{S}$ is a subset of $\mathbb{F}$. A predicate is given by a propositional formula with literals of the form $l_{i,j}$ where $l_{i,j}$ means that the feature $f_i$ takes the value $j$ in domain $D_i$.
% carte
%\todomartin{A-t-on vraiment besoin de la définition d'une formule propositionnelle?} 
A state $s\in S$ is an interpretation in the language based on the vocabulary $(l_{i,j})_{i\in[1,n], j\in D_i}$ such that $s(l_{i,j})=True$ if and only if the feature $i$ as the value $j$ in the state $s$. %$s\models d$ is defined by induction on the form of $d$: $s\models l_{x,y}$ iff $s(l_{x,y})=True$;  $s\models \neg d$ iff $s\not \models d$; $s\models d_1\vee d_2$ iff ($s\models d_1$ or $s\models d_2$); $s\models d_1\wedge d_2$ iff ($s\models d_1$ and $s\models d_2$).} 
It follows that the predicate $d$ can be evaluated in linear time for a given state.
The importance score represents the benefit of performing an action $a$ from $s$ rather than another action $a' \in A(s)\backslash\{a\}$, where this benefit is the probability of reaching a state at a horizon of $k$ that satisfies $d$. To evaluate an action we first require the notion of utility of a set of (state, probability) pairs.

% Utility
\begin{definition}[utility]
Given a predicate $d$, \emph{the utility $u_d$ of a set of (state, probability) pairs} $S$ is: $$u_d(S) =    \sum_{(s,pr)\in S, s\models d} pr$$ 
\end{definition}

% Importance score
Finally, the importance of an action $a$ from a state $s$ is defined by:% as follows.

\begin{definition}[importance]
Given a predicate $d$, a policy $\pi$, a transition function $p$,  
the \emph{importance score of $a$} from $s$ at horizon $k$ is:% defined by:
%\begin{multline*}
% imp_{d,\pi,p}^{k}(s, a) = u_d(succ_{\pi,p}^{k}(S_{(s,a)})) - \\ \underset{a' \in A(s)\backslash\{a\}}{\avg} u_d(succ_{\pi,p}^{k}(S_{(s,a')}))   
%\end{multline*}
\footnotesize $$imp_{d,\pi,p}^{k}(s, a) = u_d(succ_{\pi,p}^{k}(S_{(s,a)}))\,\, -\!\!\!\!\!\!\underset{a' \in A(s)\backslash\{a\}}{\avg} u_d(succ_{\pi,p}^{k}(S_{(s,a')}))$$\normalsize
where $\avg$ is the average and $S_{(s, a)}$ is the support of $p(.|s, a)$.\footnote{i.e. $S_{(s, a)} = \left\{ (s', p(s'|s, a)) \;\middle\vert\;
    p(s'|s, a) \neq 0 \right\}$}
\end{definition}
% Transition to approximate HXP
The importance score lies in the range $[-1, 1]$, where a positive (negative) score denotes an important (resp. not important) action in comparison with other possible actions. Its computation is $\#$W[1]-hard~\cite{saulieres:hal-04170188}, so it is necessary to approximate it, in particular by generating only part of the length-$k$ scenarios with the \emph{succ} function.%\todoinflo{c’est la première fois que le mot scenario apparaît, il n’est pas défini! le mot histoire est déjà apparu mais n’a pas non plus été défini, il me semble qu’il sont utilisés de façon interchangeable dans ce papier, dans ce cas n’en garder qu’un ? ou bien y a-t’il une différence ?}

% n-last approx. HXP
%Approximate HXP consists in considering the last $n$ time steps as deterministic over a horizon $k$, taking only the most probable transition into account. Thus, with $b$ denoting the maximum number of transitions from a state-action couple ($s, a$) of a given RL problem, only at most $b^{k-n}$ scenarios are examined.

% Transition to next section
To handle long histories on problems where the number of possible transitions is large, i.e. large horizon $k$ and large branching factor (denoted $b$ hereafter), it is necessary to use approximate methods to provide explanations in reasonable time, at the expense of only approximating the importance scores.  
In the next section, we propose a new way of explaining histories %computing HXP 
in a step-by-step backward approach, which allows us to provide explanations in reasonable time for long histories, without having to approximate the calculation of scores. 
As we will see, this leads to other computational difficulties. The result is thus a novel method for the explanation of histories with different pros and cons compared to forward-based history explanation.

%%%%%%%%%%%%%%%%%%%%%%%%%%%%%%%%%%%%%%%%%%%%%%%%%%%%%%%%%%%%%%%%%%%%%%%%
\section{Backward HXP (B-HXP)}
%%%%%%%%%%%%%%%%%%%%%%%%%%%%%%%%%%%%%%%%%%%%%%%%%%%%%%%%%%%%%%%%%%%%%%%%

% Intro
The idea of B-HXP is to iteratively look for the most important action in the near past of the state that respects the predicate under study. When an important action is found, we look at its associated state $s$ to define the new predicate to be studied. The process then iterates treating $s$ as the final state with this new predicate. Indeed, by observing only a subset of the actions in the history (near past), the horizon for calculating importance scores is relatively small. In this sense, importance scores can be calculated exhaustively. 
The predicate is then modified so that actions can be evaluated with respect to a predicate that they can achieve within a shorter horizon.
% Example
The following example will be used throughout this section to illustrate the method.

\begin{example} \label{ex:Bob}
Consider the end of Bob's day. The history of Bob's actions is: [work, shop, watch TV, nap, eat, water the plants, read]. Bob's state is represented by $5$ binary features: hungry, happy, tired, fridge, fuel. Fridge and fuel means respectively that the fridge is full and that the car's fuel level is full. Bob's final state is: ($\neg$hungry, happy, tired, $\neg$fridge, $\neg$fuel) (for the sake of conciseness, Bob's states are represented by a boolean 5-tuple. Thus, Bob's last state is: $(0,1,1,0,0)$). 
The environment is deterministic and the predicate under study is ``Bob is not hungry". We are looking for the most important actions for Bob not to be hungry. Starting from the final state, the most important action in the near past is `eat'. We are interested in its associated state, i.e.
        the state in the history before doing the action `eat', which is assumed to be $(1,0,0,1,0)$. The new predicate deduced from this state is ``Bob is hungry and has a full fridge". In the near past of $(1,0,0,1,0)$, the `shop' action is the most important one (among work, shop, watch TV and nap) for respecting this new predicate. To sum up, we can say that the reason that Bob is not hungry in the final state 
is that he went shopping (to fill his fridge) and then ate.
\end{example}

% Notation
Before describing the B-HXP method in detail, we introduce some notations. $H = (s_0, a_0, s_1, ..., a_{k-1}, s_k)$ denotes \emph{a length-$k$ history}, with $H_i=(s_i,a_i)$ denoting the state and action performed at time $i$, and for $i<j$, $H_{(i, j)}$ denotes the sub-sequence $H_{(i, j)} = (s_i, a_i, ..., s_j)$. 
%In this section, we employ the term `utility of a state $s$' to express the utility of the agent's action associated with $s$
%(this action being unique since we assume a deterministic policy).
To define the near past of a state in $H$, it is necessary to introduce the maximum length of sub-sequences: $l$. This length must be sufficiently short to allow importance scores to be calculated in a reasonable time. The value of $l$ depends on the RL problem being addressed, and specifically on the maximal number of possible transitions from any observable state-action pair, namely $b$. It follows that the lower $b$ is, the higher $l$ can be chosen to be.

% PAXp for classifiers
To provide explanations for long histories, we need a way of defining new intermediate predicates (such as Bob is hungry and the fridge is full in Example~\ref{ex:Bob}). For this we use Probabilistic Abductive eXplanations, shortened to PAXp~\cite{DBLP:journals/ijar/IzzaHINCM23}. 
The aim of this formal explanation method is to explain the prediction of a class $c$ by a classifier $\kappa$ by providing an important set of features among $\mathcal{F}$. Setting these features guarantees (with a probability at least $\delta$) that the classifier outputs class $c$, whatever the value of the other features. 
A classifier maps the feature space into the set of classes: $\kappa: \mathbb{F} \rightarrow \mathcal{K}$. We represent by $\mathbf{x}=(x_1, ..., x_n)$ an arbitrary point of the feature space and $\mathbf{v}=(v_1, ..., v_n)$ a specific point, where each $v_i$ has a fixed value of domain $D_i$.  \cite{DBLP:journals/ijar/IzzaHINCM23} defines a weak PAXp as a subset of features for which the probability of predicting the class $c=\kappa(\mathbf{v})$ is above a given threshold $\delta$ 
when these features are fixed to the values in $\mathbf{v}$. A PAXp is simply a subset-minimal weak PAXp.

% Definitions
\begin{definition}[PAXp \cite{DBLP:journals/ijar/IzzaHINCM23}] \label{def:PAXp}
Given a threshold $\delta \in [0,1]$, a specific point $\mathbf{v} \in \mathbb{F}$ and the class $c \in \mathcal{K}$ such that $\kappa(\mathbf{v}) = c$,  $\mathcal{X} \subseteq \mathcal{F}$ \emph{is a weak PAXp} if:
 $$Prop(\kappa(\mathbf{x}) = c \mid \mathbf{x}_\mathcal{X} = \mathbf{v}_\mathcal{X}) \ge \delta$$ 
where $\mathbf{x}_\mathcal{X}$ and $\mathbf{v}_\mathcal{X}$ are the projection of $\mathbf{x}$ and $\mathbf{v}$
onto features $\mathcal{X}$ respectively and $Prop(\kappa(\mathbf{x})=c \mid \mathbf{x}_\mathcal{X} = \mathbf{v}_\mathcal{X})$
is the proportion
of the states $\mathbf{x} \in \mathbb{F}$
satisfying %with 
$\mathbf{x}_\mathcal{X} = \mathbf{v}_\mathcal{X}$,
that the classifier maps to $c$, in other words %i.e.
$|\{\mathbf{x} \in \mathbb{F} \mid \mathbf{x}_\mathcal{X}{=}\mathbf{v}_\mathcal{X}$ and $\kappa(\mathbf{x}){=}c\}| / |\{\mathbf{x} \in \mathbb{F} \mid \mathbf{x}_\mathcal{X}{=}\mathbf{v}_\mathcal{X}\}|$.

The \emph{set of all weak PAXp’s} for $\kappa(\mv)=c$ wrt the threshold $\delta$ is denoted WeakPAXp$(\kappa, \mv,c, \delta, \F)$.

$\mathcal{X} \subseteq \mathcal{F}$ \emph{is a PAXp} if it is a subset-minimal weak PAXp. 
The \emph{set of all PAXp’s} for $\kappa(\mv)=c$ wrt the threshold $\delta$ is denoted  PAXp$(\kappa, \mathbf{v},c, \delta, \F)$.
\end{definition}

% Why do we use PAXp?
The idea is to use PAXp's to redefine the predicate to be studied for the next sub-sequence as we progress backwards. 
In order to fit into the PAXp framework we define the classifier $\kappa_{s,\pi,p,d,k}$ as a binary classifier based on the utility of the state $s$. We note $u^k_{d,\pi,p}(s)=u_d(succ^k_{\pi,p}(\{(s,1)\}))$, \emph{the utility of} $s$ wrt $d$ given an horizon $k$, a policy $\pi$ and a transition function $p$.
The class $\kappa_{s,\pi,p,d,k}(\mathbf{x})$ of any state $\mathbf{x}$ is the result of a comparison between the utility of $s$ and the utility of $\mathbf{x}$ for the respect of $d$.

\begin{definition}[B-HXP classifier]\label{def:classifier}
Given a state $s$, a policy $\pi$, a transition function $p$, a predicate $d$ and a horizon $k$.
The \emph{B-HXP classifier}, denoted $\kappa_{s,\pi,p,d,k}$, is a function such that:
for all $\mathbf{x}\in \mathcal{S}$, 
$$\kappa_{s,\pi,p,d,k}(\mathbf{x})= \left\{\begin{array}{ll}
True & \mbox{ if } u^k_{d,\pi,p}(\mathbf{x})\geq u^k_{d,\pi,p}(s)\\
False & \mbox{ otherwise }\end{array}\right.$$ 
\end{definition}

This classifier is specific to B-HXP. The utility threshold value depends on the state $s$ which is the state associated with the most important action in the sub-sequence studied. It is used to generate a predicate $d'$ which reflects a set of states at least as useful as $s$ 
(with a probability of at least $\delta$) for the respect of $d$. The predicate $d'$ can then be seen as a sub-goal for the agent in order to satisfy $d$.

To assess whether a subset $\X \subseteq \mathcal{F}$ is a PAXp, it is necessary to calculate the utility of each state having this subset of features, which involves using the agent's policy $\pi$ and the environment transition function $p$. A weak PAXp is then a sufficient subset of state features which ensures that a state utility is greater than or equal to %$z$ 
the utility of $s$ with probability at least $\delta$.
% How to define a new predicate using PAXp?
The new predicate is defined as the disjunction of every possible PAXp from s. 
\begin{definition}\label{def:predicate}
Given a state $s =(s_1,\ldots, s_n) \in \mathcal{S}$, a B-HXP classifier $\kappa$ on $\mathcal{S}$,  \emph{the predicate PAXpred associated with} $s$ for a given threshold $\delta$ is:
		$$\text{PAXpred}_{\kappa}(s,\delta) =  \bigvee_{\mathcal{X} \in \text{PAXp}(\kappa, s, True, \delta, \mathcal{S})} \left(\bigwedge_{f_i \in \mathcal{X}} f_i = s_i\right)$$
\end{definition}
    
\begin{continueexample}{ex:Bob}
    \begin{itshape}
    For this history, $\delta$ is set to $1$ and $l$ is $4$. The first sub-sequence studied is: [nap, eat, water the plants, read]. The most important action relative to the achievement of ``Bob is not hungry" is `eat', with a score of, say, $0.5$ and its associated state, say, $(1,0,0,1,0)$. We extract a PAXp, which includes the features \emph{fridge} and \emph{hungry} set to $1$. This means that, whatever the values of the other features, a state with \emph{fridge} and \emph{hungry} set to $1$ has (with $100\%$ probability) a utility greater than or equal to $0.5$. Indeed, this is the only PAXp, so PAXPred is (\emph{fridge}=1 $\land$ \emph{hungry}=1).
    Now, we study the sub-sequence [work, shop, watch TV, nap], in which the most important action to achieve the new intermediate predicate is `shop'.
    \end{itshape}
\end{continueexample}

    % C4 B-HXP
    \begin{figure*}
        \centering
        \includegraphics[scale=0.04]{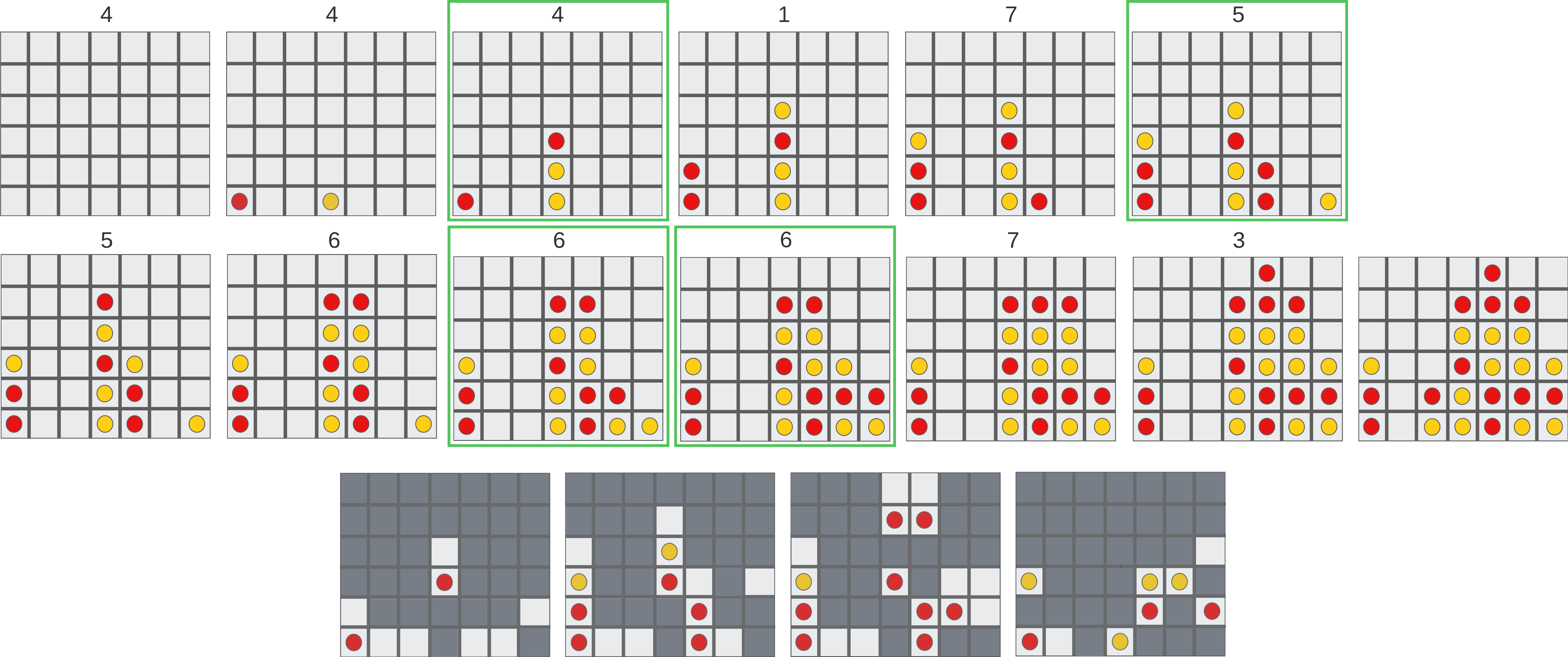}
        \caption{B-HXP for the \emph{win} predicate in the C4 problem. Above: an input history of 13 states $s_0,\ldots, s_{12}$ and 12 moves (where a move is the choice of a column by the agent (yellow) to which the environment (red) responds). Below: the predicates
        found by B-HXP corresponding to the four important moves it finds,
        each highlighted by a green frame in the history. \\ }
        \label{fig:C4_example}
    \end{figure*}

% Utility of predicate change
In the backward analysis of $H$, the change of predicate allows us to look at a short-term objective to be reached, thus keeping the calculation of HXP reasonable.
% Method description
Our method is explained in pseudo-code in Algorithm~\ref{alg:backward_HXP}. This algorithm allows us to go backwards through the history $H$, successively determining in each sub-sequence studied, the important action and its associated state predicate. The $\argmax$ function is used to find, in a given sub-sequence $H_{(i,j)}$, the most important action $a$, its associated state $s$, and its index in $H$. The latter is used to determine the next sub-sequence to consider. The \emph{PAXpred} function is used to generate the new predicate to study in the next sub-sequence, based on Definition~\ref{def:predicate}. The process stops when all actions have been studied at least once, or when the utility of the most important action in the current sub-sequence is $0$. 
Finally, the algorithm returns a list of important actions and the different predicates found.

    \begin{algorithm}
    \caption{B-HXP algorithm}\label{alg:backward_HXP}
    \textbf{Input: }\text{history $H$, maximal sub-sequence length $l$,}\\ 
    \text{agent's policy $\pi$, predicate $d$, transition function $p$,} \\ \text{probability threshold $\delta$, state space $\mathcal{S}$}
    \\
    \textbf{Output: }\text{important actions $A$, predicates $D$}
    \begin{algorithmic}
    % Init
    \State $A \gets []$ ;  $D \gets []$; $ u \gets 1$ %+ \infty$
    \State $i_{max} \gets len(H)$ ; $i_{min} \gets \max(0, i_{max} - l)$
    % Loop
    \While{$i_{min}\neq0 \text{ and } u\neq0$}
        \small\State 
          $i,s,a\gets$ $\argmax_{\tiny\begin{array}{c}
          i\in [i_{min}, i_{max}]\\
          (s,a)=H_{i}\end{array}} imp^l_{d,\pi,p} (s,a)$ 
        \State $u \gets u^l_{d,\pi,p}(s)$
        \State $d \gets \text{PAXpred}_{\kappa_{s,\pi,p,d, l}} (s, \delta)$ 
        \State $A.append(a)$ ; $D.append(d)$
        \State $i_{max} \gets i$ ; $i_{min} \gets \max(0, i_{max} - l)$
    \EndWhile
\\
    \Return $A, D$
    \end{algorithmic}
    \end{algorithm}
    
% Worst-case: all actions are displayed
With B-HXPs, it is interesting to note that the number of actions to be presented to the user is not fixed. In the worst-case scenario, a user could end up with an explanation that refers to all actions as important. This would happen if it was always the last action in each subsequence which is the most important for achieving the current predicate. However, this problem was not observed in our experiments.

% Complexity
%\flo{pour papier ECAI (ou autre): faire une propriété sur la complexité du pb de renvoyer une séquence action prédicats importants étant donné $\delta$, $l$ et $k$ en fonction de $\pi,p$ et $b$}
    % introduction
B-HXP provides the user with an explanation even for long histories. Our motivation behind B-HXP is to reduce the complexity of calculating important actions in comparison with forward HXP.
In the remainder of this section, we justify the choices made in defining B-HXP by analysing in detail the theoretical complexity of certain subproblems encountered by HXP and B-HXP.
    % forward HXP complexity
\begin{proposition}
     Given a policy $\pi$, a transition function $p$ %of max combined branching factor $b$ 
     and predicate $d$,
     the importance score computation of an action $a$ at horizon $k$ from any state $s$ for the respect of $d$, $imp^k_{d,\pi,p}(s,a)$, is \#P-hard. 
      %\flo{c’est pas plutôt in $O(b^k.|d|)$ ?}    
\end{proposition}
% carte 
%\todoinflo{j’ai toujours un pb avec ``polynomial-time verifiable predicate $d$’’: polynomial en fonction de quoi? la vérification d’un prédicat c’est $2^m$ pour une formule avec $m$ connecteurs, pourquoi ne pas donner la complexité paramètrée par la complexité de la vérification de $d$}
%\todomartin{On exige simplement que l'on puisse évaluer $d(s)$ en temps polynomial étant donné $s$.}
\begin{proof}
    This complexity is a direct consequence of the result in~\cite{saulieres:hal-04170188}: 
    given the length of the search horizon $k$ as a parameter, calculating the importance of an action is \#W[1]-hard.
    %\flo{The problem of determining the existence of a state reachable after $k$ steps that satisfies a predicate $d$ is W[1]-hard \cite{saulieres:hal-04170188}. Hence calculating the importance of an action, which requires to sum the probabilities of such states} is \#W[1]-hard.
%\todoinflo{il faut calculer l’utilité de tous les états atteints à horizon $k$, with $b$ denoting the branching factor, ça fait $b^k$ states, qui demande de faire une somme pour tous les états qui satisfont $d$ (vérif que $s\models d$ linéaire par rapport à la taille de $d$) puis moyenne pour toutes les autres actions des utilités des états obtenus après $k$ pas de temps : en résumé pour tous les états à $b^k$ faire une vérif de $d$: donc $O(b^k.|d|)$ ?}
\end{proof}

    % B-HXP importance score computation is polynomial
To avoid this computational complexity, the search horizon $l$ 
is chosen to be a small constant in the B-HXP calculation of importance scores. 
Using $b$ to denote the maximum number
of successor states from any given state 
(i.e. $b = \max_{s \in \mathcal{S}} |\{s' \in \mathcal{S} :
p(s | s,\pi(s)) > 0\}|$), there are a maximum of $b^l$ scenarios generated, and hence we have the following lemma. 

\begin{lemma}\label{lemma:imp_score}
Given a policy $\pi$ and a transition function $p$,
we assume that $b$, the maximum number of successor states
is a polynomial function of the instance size (i.e. the number of bits
necessary to specify the history together with $\pi$, $p$ 
and the predicate $d$). For %a given predicate $d$ and
a constant search horizon $l$, the computation of the importance score of any action $a$ from any state $s$ for the respect of $d$ ($imp^l_{d,\pi,p}(s,a)$)
and the utility of the state $s$ are in time
which is polynomial in the size of the instance.
\end{lemma}
% carte

%Thus, the importance score is theoretically calculated in polynomial time.
    % PAXpred output is exponential in the size of s
However, in comparison with forward HXP, it is necessary to compute intermediate predicates, in particular with PAXpred (Definition~\ref{def:predicate}). %We therefore need to look at its complexity. 
It is interesting to note that the B-HXP classifier (Definition~\ref{def:classifier}) $\kappa$ used in PAXpred 
can be evaluated in polynomial time because it is based on the calculation of the utility of a state (the main component in the importance score computation). Unfortunately, from a state $s$, the number of PAXp's according to $\kappa$, can be exponential in the size of $s$.

\begin{lemma}\label{lemma:PAXpred}
    Given a state $s$, in the worst case, PAXpred returns a predicate of size which may be exponential in the size of $s$. 
    %\todoinflo{size of $s$ c’est le nombre de features dans le langage, au pire PAXpred donne tous les états possibles donc ça ferait $|F|\times|D|$ états, chaque état décrit par $F$ valeur pour les features donc $O(|F|^2\times|D|)$ }
\end{lemma}

To support this assertion, consider the following example.

\begin{example}
A state consists of $n$ features $f_1,\ldots,f_n$,
initially all set to 0. An agent's $i$th action consists
in assigning a value to $f_i$ from its domain $D_i$,
where $D_i=\{0,1\}$ ($i=1,\ldots,n{-}1$) and $D_n=\{0,1,n\}$.
The transition function is deterministic and the agent obtains a reward only by reaching any state $s^{goal}$ such that $\sum_{i=1}^{n} f_i \geq n + \frac{n{-}1}{2}$. Let the predicate $d$ be \emph{goal}, i.e. a predicate checking whether the agent reaches a state $s^{goal}$. 
%Let the predicate $d$ be $\sum_{i=1}^{n} f_i \geq n + \frac{n{-}1}{2}$.
Consider the state $s=(1,\ldots,1,0)$ after the agent has
made $n-1$ assignments. Clearly, assigning $n$ to $f_n$
establishes the predicate in one step. Moreover, any state with at least (n-1)/2 assignments of 1 among the first $n-1$ features, allows us to attain the goal in one step by assigning $n$ to the last feature $f_n$. Hence, the PAXp's of this predicate
$d$ from $s$ and for $l=1$, $\delta=1$ are precisely the subsets of 
$(n{-}1)/2$ literals of the form $f_i=1$ ($1 \leq i \leq n{-}1$).
There are $n{-}1 \choose (n{-}1)/2$ such PAXp's, so the corresponding predicate
PAXpred (for $\delta=1$) is of exponential size.
\end{example}

To exhaustively determine an intermediate predicate \emph{PAXpred}, therefore, exponential space in the size of the state $s$ is required.
    % Approximation of PAXpred : one AXp
One approximation to \emph{PAXpred} is to calculate only one AXp (i.e. a PAXp with $\delta = 1$). This reduces the problem to a more amenable co-NP problem~\cite{DBLP:journals/ai/CooperS23} where only a counterexample state $s'$ is needed to show that the set of features $\mathcal{X} \subseteq \mathcal{F}$ is not an AXp.
Unfortunately, this approximation yields predicates that are often too specific because of $\delta = 1$. 
    % Approximation of PAXpred: one LmPAXp
Instead, in order to provide sparser intermediate predicates, the considered approximation consists in computing one PAXp, or more precisely one \emph{locally-minimal} PAXp. \emph{Locally-minimal} PAXp's is a particular class of weak PAXp which are not necessarily subset-minimal. Formally, a set of features $\mathcal{X} \subseteq \mathcal{F}$ is a \emph{locally-minimal} PAXp if 
$\X\in\text{WeakPAXp}(\kappa, \mathbf{v}, c, \delta,\F)$ and 
    for all $j \in \mathcal{X}$, $\X\!\setminus\!\{j\}\not\in\text{WeakPAXp}(\kappa, \mathbf{v}, c,\delta,\mathbb{F})$.

The \emph{findLmPAXp} algorithm~\cite{DBLP:journals/ijar/IzzaHINCM23} is used to calculate a \emph{locally-minimal} PAXp.
    % Comparison HXP / B-HXP
Although the approximation consists of calculating just one LmPAXp, it remains a hard counting problem.
We first prove hardness before explaining why this is nevertheless an improvement over the hardness of forward HXP.

\begin{lemma}\label{lemma:LmPAXp}
    Given a state $s$, the computation of 
    a locally-minimal PAXp is in FP$^{{\mathrm \#P}}$ 
    (i.e. in polynomial time using a \#P-oracle)
    and determining whether a given subset of features is 
    a locally-minimal PAXp is \#P-hard.
\end{lemma}

\begin{proof}
    % Introduction: FP using #P oracle
    We first show the inclusion in FP$^{{\mathrm \#P}}$.
    From a specific state $s$, the search for a locally-minimal PAXp $\mathcal{X} \subseteq \mathcal{F}$ begins by initialising $\mathcal{X}$ to $\mathbf{F}$ 
    (the trivially-correct explanation consisting
    of all the features). Then, the following test is carried out for a feature $f_i \in \mathcal{X}$: if $f_i$ is removed from $\mathcal{X}$, does $\mathcal{X}$ remain a WeakPAXp? If so, $f_i$ is removed, otherwise retained. This test is performed for each feature of $\mathcal{X}$.
	Determining whether a subset $\mathcal{X} \subseteq \mathcal{F}$ is a WeakPAXp %(with $\delta < 1$) 
    amounts to counting the number of states $\mathbf{x}$ which match $s$ on $\mathcal{X}$ (i.e. $\mathbf{x}_\mathcal{X} = s_\mathcal{X}$)
    that are classified as True (i.e. such that $\kappa_{x,\pi,p,d,k}=$True). Lemma~\ref{lemma:imp_score}
    tells us that $\kappa$ can be evaluated in polynomial time,
    so WeakPAXp belongs to \#P.
    The search for a locally-minimal PAXp can thus be performed in polynomial time in the size of $s$, i.e. its number of features, using a \#P-oracle WeakPAXp.
 
    % Reduction
        % Introduction & Graph description
    To prove \#P-hardness, it suffices to give a polynomial reduction from {\sc perfect matchings} which is a \#P-complete problem~\cite{DBLP:journals/siamcomp/Valiant79}.
    %The feature space $\F$ is represented by 
    Consider a graph $G=(V,E)$ where $V$ is the set of vertices and $E$ is the set of edges. $V$ is decomposed into two parts $V_1$ and $V_2$ so that each edge has one end in $V_1$ and another in $V_2$. In other words, $G$ is a bipartite graph. A perfect matching is a set of edges such that no pairs of vertices has a common vertex.
    
        % Description of G
    %\st{$V_1$ represents}
    Let us define a set of $n$ features $f_i \in \mathcal{F}$ such that each feature $f_i$ corresponds to a vertex, say $i$, in $V_1$. We define the domain $D_i$ of each feature $f_i$ as the set of vertices in $V_2$ that are related to $i$ by $E$.
	%$V_2$ represents, for each feature $f_i \in \mathcal{F}$, its possible values $v_{ij}$.
    %\flo{\st{The joint domain $D$ of features $f_i$ is the union of the domains $D_i$, i.e. $D  = \bigcup_{i=1}^n D_i$. This is described by $V_2$, where each vertex represents a value $v_j \in D$.}} \st{Thus}
	In other words, an edge $e \in E$ is a possible assignment of a feature $f_i$ to a value $v_{j}$.
        % G is used to represent any state
    The feature space $\F$ is modeled by $G$. Indeed, each point $\mathbf{x}=(f_1, \dots, f_n) \in \F$ can be obtained by assigning each feature in $V_1$ to a value in $V_2$. Thus, a state $\mathbf{x}$ is represented by a set of edges $E_\mathbf{x} \subseteq E$. 
        % kappa classifier
    To compute WeakPAXp, the following classifier is used: 

    $$\kappa(\mathbf{x})= \left\{\begin{array}{ll}
    True & \mbox{ if } \forall i \neq j, f_i \neq f_j\\
    False & \mbox{ otherwise }\end{array}\right.$$
    
    This classifier outputs True for any state $\mathbf{x}$ comprising a distinct valuation for each feature. In $G$, such a state $\mathbf{x}$ is represented by a set of edges $E_\mathbf{x} \subseteq E$ in such a way that each pair of edges in $E_\mathbf{x}$ has no common vertices. Thus, $E_\mathbf{x}$ is a perfect matching.
        % Specific point v and \mathcal{X}
    From a specific point $\mathbf{v}$ such that $\kappa(\mathbf{v}) = True$, we are interested in knowing whether $\mathcal{X} = \emptyset$ is a weakPAXp, for different values of $\delta$. Based on Definition~\ref{def:PAXp}, it is necessary to calculate the proportion of states classified as True to determine whether $\mathcal{X}$ is a weakPAXp.
        % How to answer the weakPAXp question
    With $\mathcal{X} = \emptyset$, the total number of states that match $\mathcal{X}$ corresponds to the number of states in $\F$. This is easily obtained by multiplying the domain-sizes $|D_i|$, $|\F| = \Pi_{i=1}^n |D_i|$. %Note that the domains $D_i$ can be obtained from $G$ by retrieving the set of adjacent vertices for each vertex $f_i$.
	The number of states which match $\mathcal{X}$ and are classified by $\kappa$ to True, is obtained by counting the states which have a distinct valuation for each feature, i.e. by counting the perfect matchings in $G$.
 $\mathcal{X}=\emptyset$ is a weakPAXp iff it is a (locally-minimal) PAXp. 
 Thus, we have reduced the {\sc perfect matchings} problem to a polynomial number of calls to the locally-minimal PAXp problem. 
 Hence there is a polynomial-time Turing reduction from {\sc perfect matchings} to the locally-minimal PAXp problem, which is therefore \#P-hard.
\end{proof}

Consequently, the B-HXP complexity is deduced from the 
complexity of the computation of importance scores and the 
generation of the intermediate predicates.

\begin{proposition}\label{prop:B-HXP}
    %Given a history $H$, a policy $\pi$, a transition function $p$ and an initial predicate d, 
    The B-HXP computation is in FP$^{{\mathrm \#P}}$ and is \#P-hard.
\end{proposition}

\begin{proof}
    From Lemma~\ref{lemma:imp_score} and Lemma~\ref{lemma:LmPAXp}, we deduce that the computational complexity of B-HXP is in FP$^{{\mathrm \#P}}$
    and is \#P-hard, in particular due to the complexity of the generation of a new predicate. 
\end{proof}

The computation of a B-HXP and the importance score computation in forward HXP are both \#P-hard. But, it is important to note that counting %a number of
states (in B-HXP) is less computationally expensive than counting %a number of 
scenarios (i.e. a sequence of $k$ states).

We can provide a finer analysis by studying fixed-parameter
tractability in $n$, the size of a state.
A problem is fixed parameter tractable (FPT) with respect to parameter $n$
if it can be solved by an algorithm running in time $O(f(n)\times N^h))$, 
where $f$ is a function of $n$ independent of the size $N$ of the instance,
and $h$ is a constant~\cite{DBLP:series/mcs/DowneyF99}.

\begin{proposition} Given a sequence of $k$ states, with each state of size $n$,  a policy $\pi$, a predicate $d$, a transition function $p$, a threshold $\delta$, a constant length $l$ and $b$, the maximum number of successor states from any given state using $\pi$ which is assumed to be polynomial in the instance's size, the complexity of finding a B-HXP is FPT in $n$.
\end{proposition}
\begin{proof}
Recall that 
the input includes a length-$k$ history made up of $k$ states
of size $n$. %Recall that we use $b$ to denote the maximum number
%of successor states from any given state using a policy $\pi$.
%Although we do not specify
%how $\pi$ and $p$ are represented in the input to the problem,
%it is reasonable to assume that $b$ is a polynomial
%function of the size of an instance. 
An exhaustive search over
all possible alternative length-$k$ scenarios would
require $\Omega(nb^k)$ time (and hence is not FPT
in $n$ because of the exponential dependence on $k$,
even if $b$ is a constant). On the other hand, 
B-HXP need only exhaust over scenarios of constant
length $l$. The complexity of redefining the predicate
by finding one LmPAXp is $f(n)b^l$ for some function 
$f$ of $n$ (the state size). This redefinition of the
predicate must be performed at most $k$ times, giving
a complexity in $O(f(n)b^l k)$. It follows that
B-HXP is FPT in $n$ since $b^l k$ is a polynomial
function of the size of an instance.
\end{proof}

%The computationally hard part of this approach is the predicate generation. Enumerating all the PAXp's turns out to be intractable. To support this assertion, even finding a single AXp (which is a PAXp with $\delta = 1$) is in general NP-hard, for example in the case of a DNF classifier~\cite{DBLP:journals/ai/CooperS23}. 
%Also, finding a single PAXp when $\delta < 1$ is NP-hard even for decision trees~\cite{DBLP:conf/nips/ArenasBOS22}. A further computational difficulty, specific to our problem, is that our classifier $\kappa_{s,\pi,p,d,k}$ requires, at each call, the computation of the action utility, which is a \#W[1]-hard problem~\cite{saulieres:hal-04170188} w.r.t. the parameter $k$.

% Approximation: Locally-minimal PAXp
%Thus, we decided to limit the definition of a predicate $d$ to the generation of one weak PAXp. To obtain a predicate $d$ in reasonable time, we need to look at a particular class of weak PAXp, the \emph{locally-minimal} PAXp's, which are not necessarily subset-minimal. Formally, a set of features $\mathcal{X} \subseteq \mathcal{F}$ is a \emph{locally-minimal} PAXp if 
%$\X\in\text{WeakPAXp}(\kappa, \mathbf{v}, c, \delta,\F)$ and 
%    for all $j \in \mathcal{X}$, $\X\!\setminus\!\{j\}\not\in\text{WeakPAXp}(\kappa, \mathbf{v}, c,\delta,\mathbb{F})$

%The \emph{findLmPAXp} algorithm~\cite{DBLP:journals/ijar/IzzaHINCM23} is used to calculate a \emph{locally-minimal} PAXp.

% Transition next section
In short, B-HXP keeps the calculation of importance scores exhaustive, by cutting the length-$k$ history into sub-sequences of length $l$, starting from the end.
The most important action of a sub-sequence is retained, and its associated state is used to define $d'$, the new predicate to study, using \emph{locally-minimal} PAXp. $d'$ is then studied in a new sub-sequence. This process is iterated throughout the history. The next section presents examples of B-HXP.

The calculation of new predicates is computationally challenging but is feasible (using certain approximations: calculating just one rather than all PAXp's and, as we will see later, the use of sampling rather than exhaustive search over feature space). The resulting method B-HXP provides a novel approach to explaining histories.

%%%%%%%%%%%%%%%%%%%%%%%%%%%%%%%%%%%%%%%%%%%%%%%%%%%%%%%%%%%%%%%%%%%%%%%%
\section{Experiments}
%%%%%%%%%%%%%%%%%%%%%%%%%%%%%%%%%%%%%%%%%%%%%%%%%%%%%%%%%%%%%%%%%%%%%%%%

% Intro
    % Problems & training
The experiments were carried out on 3 RL problems: Frozen Lake (FL),
Connect4 (C4) and Drone Coverage (DC)~\cite{DBLP:conf/icaart/SaulieresCB23}.
Q-learning~\cite{watkins1992q} was used to solve the FL problem, and Deep-Q-Network~\cite{mnih_human-level_2015} for C4 and DC.
    % Config
The agents' training was performed using a Nvidia GeForce GTX 1080 TI GPU, with 11 GB of RAM. The B-HXP examples were run on an HP Elitebook 855 G8 with 16GB of RAM %(a link to the source code will be made available in the final version).
(source code available \cite{saulieres_2024_13120510}).
    
    % Plan
The first part of this section describes the problems and the studied predicates. %(for more details, see~\cite{saulieres:hal-04170188}). 
The second part presents B-HXP examples. 
%
    % Figure
In the associated figures, the history is displayed over two lines. Each history is composed of 13 states and 12 actions. The action taken by the agent from a state is shown above it. Important actions and their states are highlighted by a green frame. The third line in figures~\ref{fig:C4_example},\ref{fig:DC_example} corresponds to the predicates generated during the B-HXP, where a dark grey cell means that this feature is not part of the predicate.
In Tables 1, 2 and 3, the importance scores given are w.r.t. either the initial predicate or the intermediate ones. 

    \begin{figure}[h]
        \centering
        \includegraphics[scale=0.06]{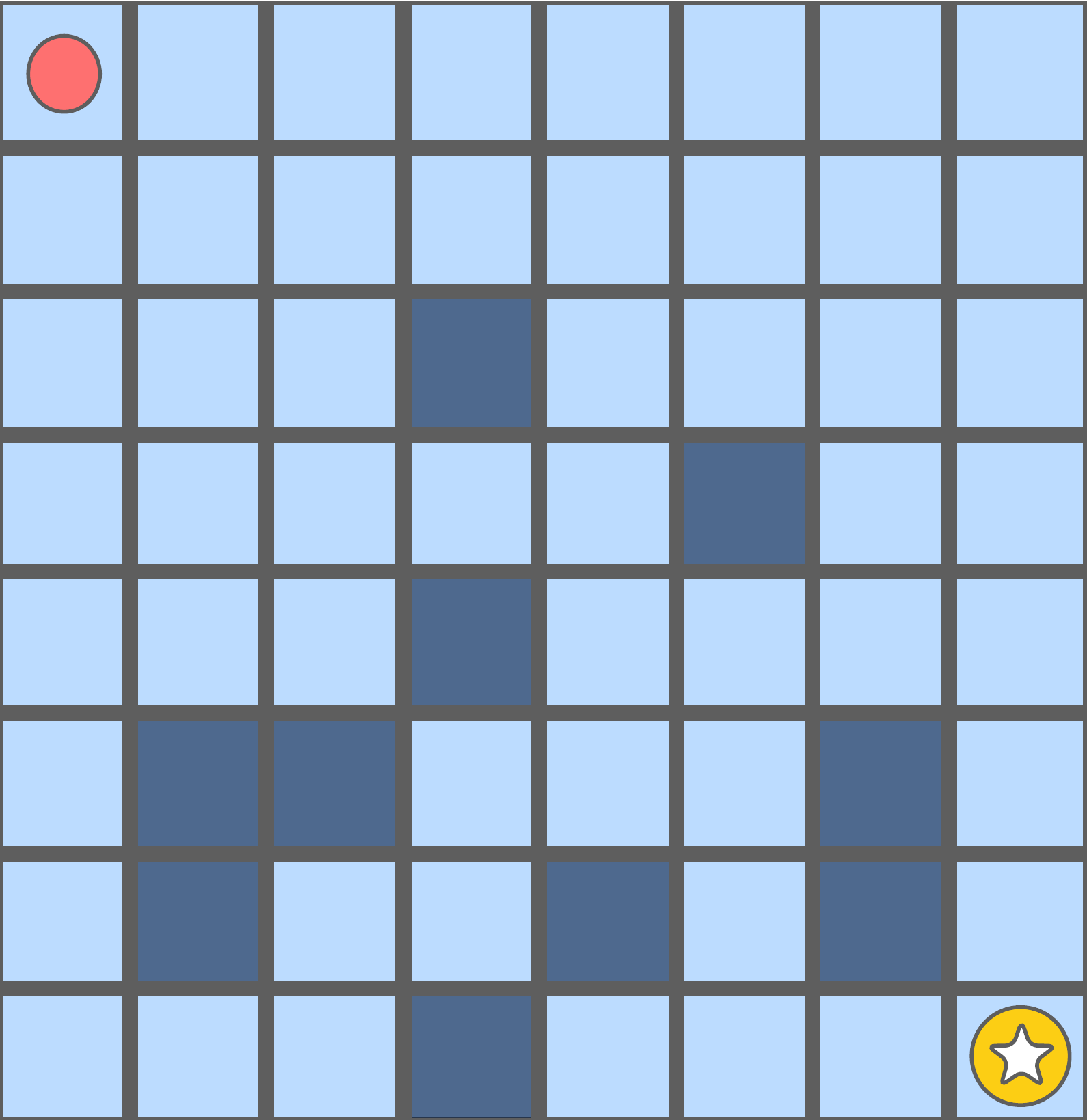}
        \caption{Frozen Lake $8\times8$ map.}
    \end{figure}
    \vspace{2mm}

    % DC B-HXP
    \begin{figure*}
        \centering
        \includegraphics[scale=0.036]{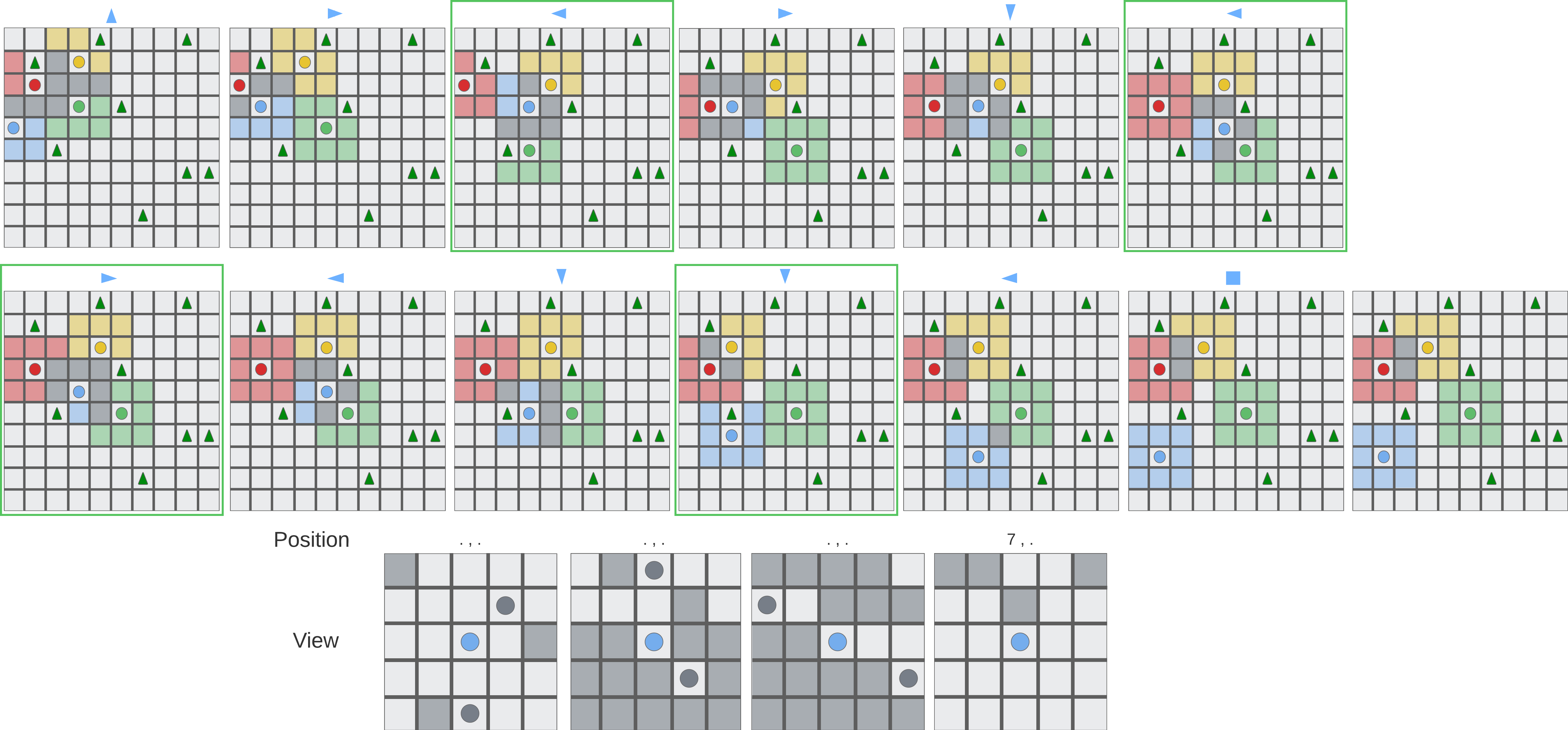} %width=0.8\textwidth
        \caption{B-HXP for the \emph{perfect cover} predicate in a DC problem history.
        %The intermediate predicates are shown in the last row. 
        The rightmost situation in the last row, corresponding to the first PAXpred predicate, states that the blue drone is in row 7 and that the light grey squares are free. The other intermediate predicates impose the relative positions of two other drones and that some squares are free. The drones are represented by dots and the trees by green triangles. \\ }
        \label{fig:DC_example}
    \end{figure*}
    \vspace{2mm}

    \subsection{Description of the problems}
    
\paragraph{Frozen Lake}
    
% Problem
In this problem, the agent moves on the surface of a frozen lake (2D grid) to reach a certain goal position, avoiding falling into the holes. The agent can move in any of the 4 cardinal directions and receives a reward only by reaching the goal position. 
However, due to the slippery surface of the frozen lake, if the agent chooses a direction (e.g. \emph{up} as in the second state of Figure \ref{fig:FL_example}), it has $0.6$ probability to go in this direction and 0.2 to go towards each remaining direction except the opposite one (e.g., for \emph{up}, $0.2$ to go \emph{left} and $0.2$ to go \emph{right}, as occurred in the scenario of Figure \ref{fig:FL_example} where the agent moved right in the third state after performing \emph{up} in the second one).
The agent's state is composed of $5$ features: its position (P) and previous position (PP) on the map, the position of one of the two holes closest to the agent (HP), the Manhattan distance between the agent's initial position and his current position (PD), and the total number of holes on the map (HN).
This last feature was added as a check: since it is a constant it should never appear in the redefined predicate, which was indeed the case.

% Predicates
Predicates \emph{win}, \emph{holes} and \emph{region} were studied. They respectively determine whether the agent reaches the goal, falls into a hole or reaches a pre-defined set of map positions.

    \begin{figure}[h]
        \centering
        \includegraphics[scale=0.08]{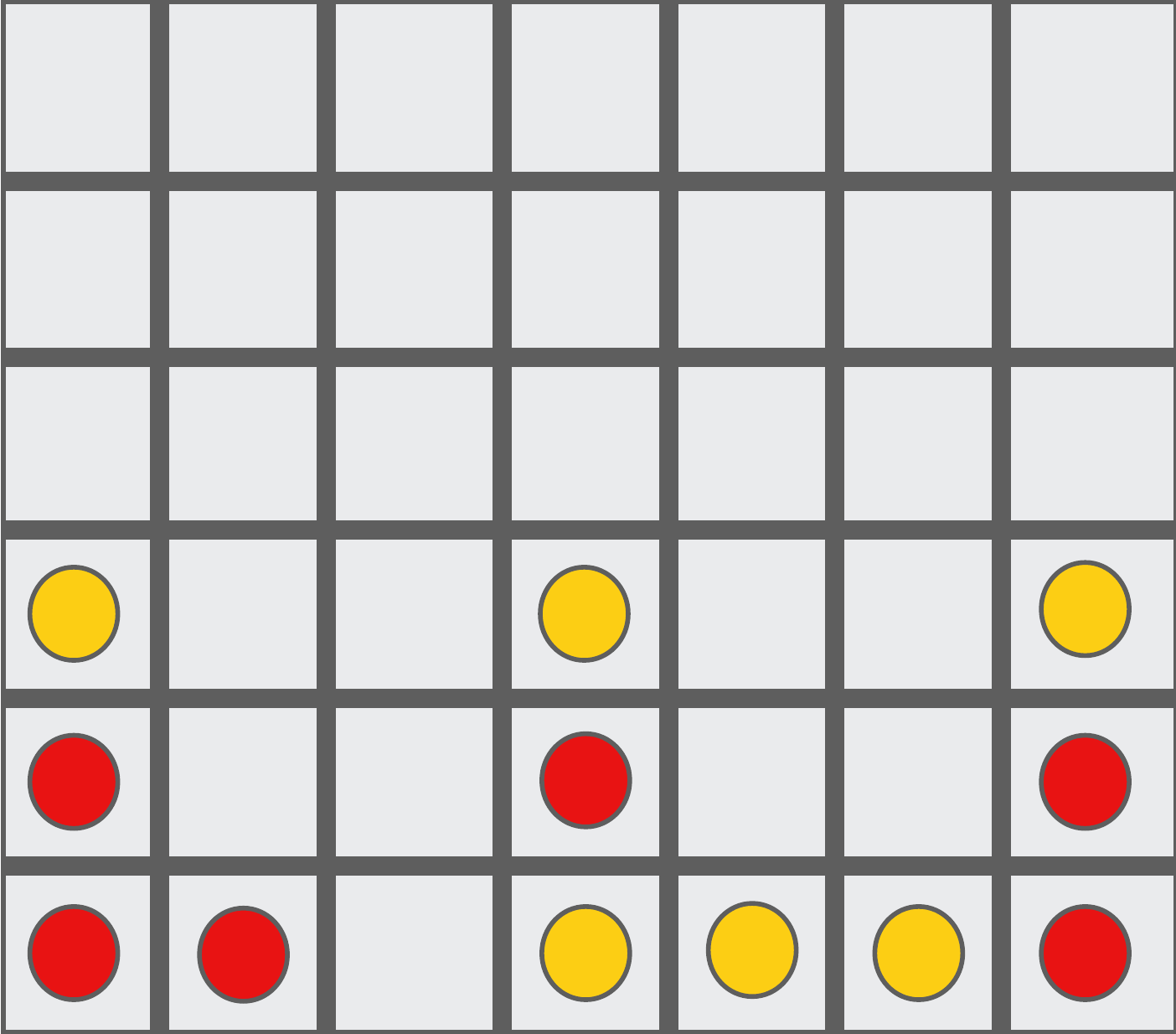}
        \caption{Connect4 board.}
    \end{figure}
    \vspace{2mm}

\paragraph{Connect4}
    
% Problem
The Connect4 game is played on a 6 by 7 vertical board, where the goal is to align 4 tokens in a row, column or diagonal. Two players play in turn.
An agent's state is the whole board. An action corresponds to dropping a token in a column. The agent receives a reward only by reaching terminal states: $1$, $-1$, $0.5$ if the state represents an agent's win, loss or draw respectively. As the agent does not know the next move of the opponent, transitions are stochastic.

% Predicates
Five predicates were studied, including the obvious \emph{win} and \emph{lose}. For the other three, the initial state is compared with the final states of the generated scenarios. The predicate \emph{control mid-column} is satisfied when the agent has more tokens in the middle column. The predicates \emph{3 in a row} and \emph{counter 3 in a row} are satisfied respectively when the agent obtains more alignments of 3 tokens on the board, and prevents the opponent from obtaining more alignments of 3 tokens on the board.

\paragraph{Drone Coverage}

    \begin{figure}[h]
        \centering
        \includegraphics[scale=0.05]{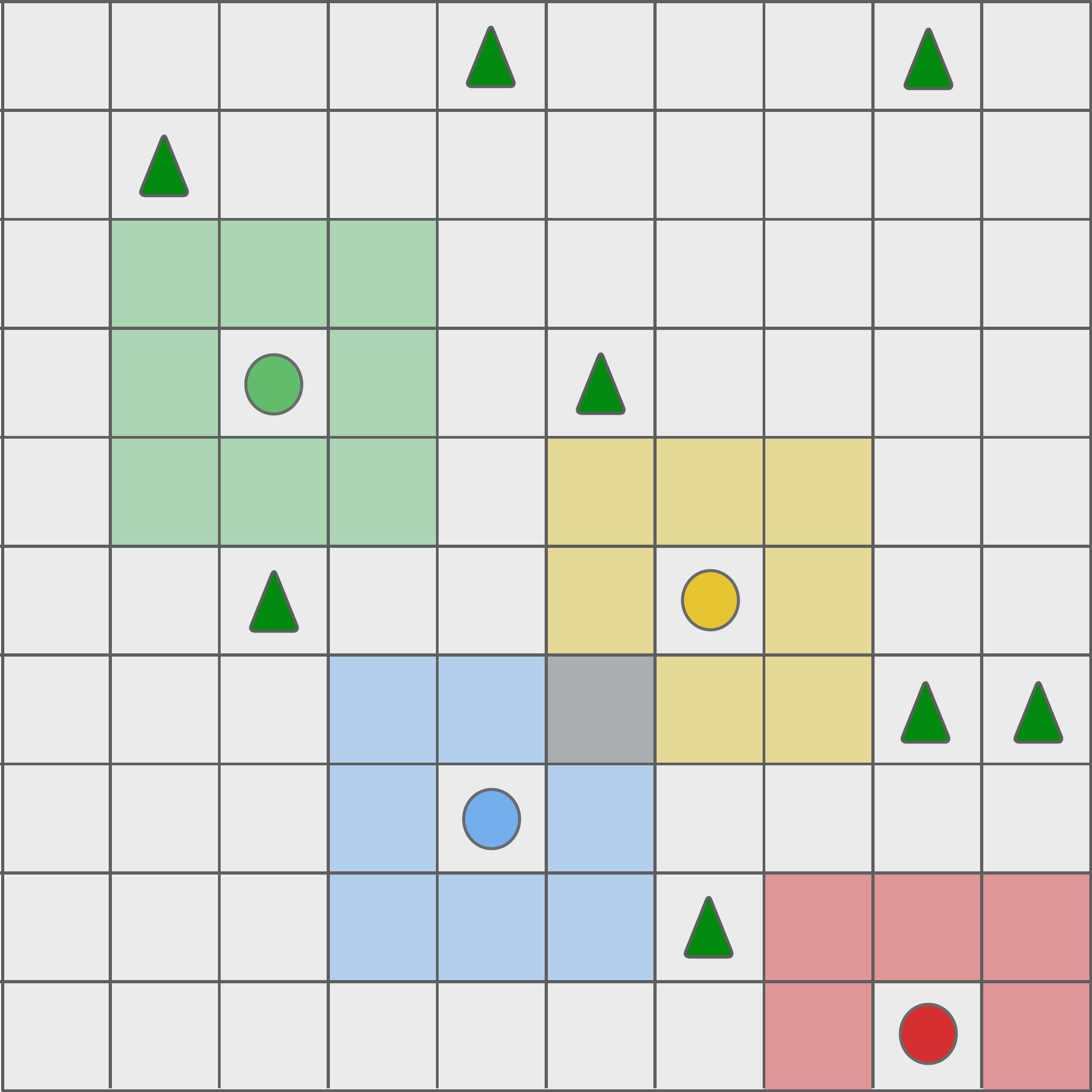}
        \caption{Drone Coverage $10\times10$ map.}
    \end{figure}
    \vspace{3mm}
    
% Problem
In this problem, four drones must cover (observe) the largest area of a windy 2D map, while avoiding crashing into a tree or another drone. A drone can move in any of the 4 cardinal directions or remain stationary. A drone's cover is the $3\times3$ square centered on it. A drone's cover is optimal when it does not contain any trees and there is no overlap with the covers of the other drones. Hence, its reward is based on its cover and neighbourhood. Moreover, it receives a negative reward in the case of a crash.
An agent's state is made up of its view, a $5\times5$ image centered on it, and its position, represented by ($x,y$) coordinates. After an agent's action, the wind pushes the agent \emph{left, down, right, up} according to the following distribution: $[0.1, 0.2, 0.4, 0.3]$ unless the action is $stop$ or the agent and wind directions are opposite (in these cases the wind has no effect).

% Predicates
Ten predicates for the DC problem were studied (local and  global versions of): \emph{perfect cover, maximum reward, no drones, crash} and \emph{region}. Local versions concern a single agent, whereas the global versions concern all agents. Local versions of predicates allow to check whether an agent reaches a perfect cover (\emph{perfect cover}), gets a maximum reward (\emph{maximum reward}), has no drones in its view range (\emph{no drones}) and did not crash (\emph{crash}).  As an example, the \emph{local no drones} predicate checks whether the agent has no drones in its view range while the \emph{global no drones} predicate checks whether each agent has no drones in its view range. The map was divided into $4$ regions of dimension $5\times 5$ for the \emph{region} predicate. 
In the local version, the predicate checks whether the agent has reached a certain region.
In the global version, each agent must be in its own distinct region.
    
    \subsection{B-HXP examples}

    % Introduction
    To provide B-HXPs in reasonable time, the $sample$ parameter, which corresponds to the maximum number of states observed for a feature evaluation in the \emph{findLmPAXp} algorithm~\cite{DBLP:journals/ijar/IzzaHINCM23}, i.e. the predicate generation, was set to $10$ in the following examples. In other words, to avoid an exhaustive search over $\mathbb{F}$, the proportion in Definition~\ref{def:PAXp} was computed based on $10$ samples.

    % FL imp. scores
    \begin{table}[h]
        \caption{Importance scores in the FL history~\ref{fig:FL_example}}
        \centering
        \setlength\tabcolsep{2pt}
        \vspace{4mm}
        \begin{tabular}{| c | c | c | c | c |}
        \hline
         {Predicate} & \multicolumn{4}{|c|}{Time-step / Importance score}\\
         \hline
         \multirow{2}{*}{\emph{win}} &  8 & 9 & 10 & 11\\ %\cline{2-5}
                                     & -0.001 & 0.04 & 0.012 & \textbf{0.114} \\
         \hline
         %\multirow{2}{*}
         {\footnotesize$\text{PAXpred}_{\kappa}(s_{11}, 0.7)$}
          &  7 & 8 & 9 & 10\\ 
           (a.k.a. \emph{purple)}                         & 0.006 & -0.008 & \textbf{0.102} & 0.087 \\
         \hline
         %\multirow{2}{*}
         {\footnotesize$\text{PAXpred}_{\kappa}(s_{9}, 0.7)$} &  5 & 6 & 7 & 8\\ 
            (a.k.a. \emph{green)}
            & -0.0003 & \textbf{0.0} & -0.001 & -0.0003 \\
        \hline
        \end{tabular}
        \label{tab:FL_imp_scores}
    \end{table}

    % Example of B-HXP: Frozen Lake
    \paragraph{Frozen Lake}

    In Figure~\ref{fig:FL_example}, the agent is symbolized by a red dot, the dark blue cells are holes and the destination cell is marked by a star. 
    Each action performed is represented by a red arrow which corresponds to the direction chosen by the agent in the state described below. 
    In the states associated with the most important actions, the genericity of the computed predicates is represented by colored cells. A predicate is said to be \emph{generic} if it is respected by a large number of different states. The first predicate is represented in purple, the second in green. A colored cell (purple or green) means that the predicate is valid for all the states whose position (P) is this cell on the grid (i.e. the states whose feature value P is this cell). 

    A B-HXP (computed in 2 seconds) for a FL history is shown in Figure~\ref{fig:FL_example}, with $l{=}4$, $\delta{=}0.7$. 
    Importance scores are presented in Table~\ref{tab:FL_imp_scores}. The \emph{right} action linked to the penultimate state {\small$s_{11} = \{P=(7,8), PP=(6,8), HP=(6,7), PD=13, HN=10\}$} is the most important in the first sub-sequence studied in order to \emph{win}. The predicate {\small$PAXpred_\kappa(s_{11}, 0.7) = (PD=13)$}, computed from %based on 
    $s_{11}$ with $\delta=0.7$, is named \emph{purple} hereafter. The states described by the predicate are shown in purple in Figure~\ref{fig:FL_example}. In the following sub-sequence, the \emph{down} action linked to state {\small$s_9 = \{P=(5,8), PP=(5,7), HP=(6,7), PD=11, HN=10\}$} is the most important to respect \emph{purple}. The predicate {\small$PAXpred_\kappa(s_9, 0.7) = (P=(5,8)) \wedge (PP=(5,7)) \wedge (HP=(6,7))$}, computed from $s_9$, is named \emph{green} (the states described are shown in green in Figure~\ref{fig:FL_example}).   
    We note that \emph{purple} describes more states than \emph{green}. The latter is not generic enough, which is reflected in the importance scores, which are close to 0: whatever the action, it is unlikely to respect this predicate after 4 time steps. The entire history is not explored when calculating the B-HXP, as the utility of the last state selected $s_6$ is 0. 
   The selected actions form a meaningful explanation when we look at the  predicates studied. However, the redefined predicates fairly quickly become very specific and probably of little help in explaining why the agent won.

    % C4 imp. scores
    \begin{table}
        \centering
        \setlength\tabcolsep{12pt}
        \caption{Importance scores in the C4 history~\ref{fig:C4_example}}
        \vspace{4mm}
        \begin{tabular}{|c| c | c | c |}
        \hline
         {Predicate} & \multicolumn{3}{|c|}{\small Time-step / Importance score}\\
         \hline
         \multirow{2}{*}{\emph{win}} &  9 & 10 & 11\\
                                     & \textbf{0.726} & 0.099 & 0.16 \\
         \hline
         \multirow{2}{*}{\small$\text{PAXpred}_{\kappa}(s_{9}, 0.8)$} & 6 & 7 & 8\\ 
                                     & -0.006 & 0.03 & \textbf{0.113} \\
         \hline
         \multirow{2}{*}{\small$\text{PAXpred}_{\kappa}(s_{8}, 0.8)$} & 5 & 6 & 7\\ 
                                     & \textbf{0.003} & 0.0 & 0.0 \\
        \hline
         \multirow{2}{*}{\small$\text{PAXpred}_{\kappa}(s_{5}, 0.8)$} & 2 & 3 & 4\\ 
                                     & \textbf{0.003} & 0.0 & 0.0 \\
        \hline
        \end{tabular}
        \label{tab:C4_imp_scores}
    \end{table}

    % Example of B-HXP: Connect4
    \paragraph{Connect4}
    In Figure~\ref{fig:C4_example}, we are interested in the actions of the agent playing the yellows tokens. The other player is considered as the environment's response. 
    Each action performed by the agent is represented by a number which corresponds to the column number in which the agent has chosen to drop a token. 
    The predicates displayed on the third line have been calculated from the states associated with the most important actions.

    With $l$ set to 3 and $\delta$ to 0.8, a B-HXP (computed in 10 seconds) for a C4 history is shown in Figure~\ref{fig:C4_example}. A large part of the board is ignored in the PAXpred predicates (see the line below the history in Figure~\ref{fig:C4_example}), which gives the user an intuition of the type of states that the agent must reach. Importance scores are presented in Table~\ref{tab:C4_imp_scores}. 
    Almost all the actions returned are related to setting up the token alignment leading to victory, which is interesting because the predicate \emph{win} is only studied on the last three states of the history. The other predicates provide actions linked to achieving a \emph{win}. 
    The first redefined predicate (the rightmost image on the third line of Figure~\ref{fig:C4_example}) describes a partial board configuration: from positions satisfying this predicate, an agent following the learnt policy has at least 80\% chance of achieving a win in the final position.
    However, as in the FL B-HXP, apart from the predicate defined at time-step 9, the other predicates generated seem to be not sufficiently generic, which can be seen in the scores. 
    Indeed, the second redefined predicate (the second-from-right image in the last line of Figure~\ref{fig:C4_example}) is so specific as to uniquely determine the exact board configuration (given the height of columns and the number of tokens played). Nevertheless, $3$ out of the $4$ important actions returned set up the winning diagonal.
   %NOT DISCUSSED IN DETAIL IN THE CONCLUSION: The value of $l$ has an important impact on the computation of importance scores. This point is explored in more detail in the conclusion. 
    
    % DC imp. scores
    \begin{table}
        \caption{Importance scores in the DC history~\ref{fig:DC_example}}
        \centering
        \setlength\tabcolsep{12pt}
        \vspace{4mm}
        \begin{tabular}{| c | c | c | c |}
        \hline
         {Predicate} & \multicolumn{3}{|c|}{\small Time-step / Importance score}\\
         \hline
         \multirow{2}{*}{\emph{perfect cover}} &  9 & 10 & 11\\
                                     & \textbf{0.114} & 0.063 & 0.056 \\
         \hline
         \multirow{2}{*}{\small$\text{PAXpred}_{\kappa}(s_{9}, 1.0)$} & 6 & 7 & 8\\ 
                                     & \textbf{0.008} & 0.006 & 0.002 \\
         \hline
         \multirow{2}{*}{\small$\text{PAXpred}_{\kappa}(s_{6}, 1.0)$} & 3 & 4 & 5\\ 
                                     & 0.053 & -0.013 & \textbf{0.09} \\
        \hline
         \multirow{2}{*}{\small$\text{PAXpred}_{\kappa}(s_{5}, 1.0)$} & 2 & 3 & 4\\ 
                                     & \textbf{0.034} & -0.036 & 0.023 \\
        \hline
        \end{tabular}
        \label{tab:DC_imp_scores}
    \end{table}

    % Example of B-HXP: DC
    \paragraph{Drone Coverage}
    
    In Figure~\ref{fig:DC_example}, we look at the actions of the blue drone. In the history, each drone is represented by a colored dot, and trees by green triangles. 
    The coverage of a drone is represented by cells of the same color as the drone. Dark grey cells mean that there is an overlap of coverage between two drones. 
    Above each environment display, only the action performed by the blue drone is shown. The actions are represented by an arrow or a square to represent, respectively, the direction chosen by the agent or the action of remaining stationary. 
    Each predicate presented on the third line of Figure~\ref{fig:DC_example} represents a partial description of a state of the blue drone in the history, including its position and its vision. In the vision part of the predicate, a dark grey box means that the feature is not part of the predicate. The rightmost situation in the last row, corresponding to the first PAXpred predicate, states that the blue drone is in row 7 and that the light grey squares are free. The other intermediate predicates impose the relative positions of two other drones and that some squares are free.

    A B-HXP (calculated in 13 seconds) for a DC history is shown in Figure~\ref{fig:DC_example}, with $l{=}3$ and $\delta{=}1$. The explanation is for the blue drone with the original predicate that this agent has a perfect cover. Importance scores are presented in Table~\ref{tab:DC_imp_scores}. The most important actions reflect the blue drone strategy %in this history 
    of moving away from the drones and trees in its cover. The predicates give a good intuition of the type of state (position and $5{\times}5$ view) the agent is trying to reach. 

%%%%%%%%%%%%%%%%%%%%%%%%%%%%%%%%%%%%%%%%%%%%%%%%%%%%%%%%%%%%%%%%%%%%%%%%
\section{Related work}
%%%%%%%%%%%%%%%%%%%%%%%%%%%%%%%%%%%%%%%%%%%%%%%%%%%%%%%%%%%%%%%%%%%%%%%%

% Introduction
XRL methods can be clustered according to the scope of the explanation (e.g. explaining a decision in a given situation or the policy in general), the key RL elements used to produce the explanation (e.g. states~\cite{DBLP:conf/icml/GreydanusKDF18,  DBLP:journals/corr/abs-1909-12969}, rewards~\cite{juozapaitis_explainable_nodate, DBLP:journals/corr/abs-2210-04723}), or the form of the explanation (e.g. saliency maps~\cite{DBLP:conf/icml/GreydanusKDF18} sequence-based visual summaries~\cite{DBLP:conf/atal/AmirA18, DBLP:journals/ai/SequeiraG20, DBLP:conf/icaart/SaulieresCB23}). 

% Surrogate Models

%\leo{Several works propose the use of surrogate models to better interpret the agent's policy. The idea is to train an interpretable model to best approximate a high-performance black-box model. 
%In this sense, Danesh et al transform a Recurrent Neural Network policy into a compact Moore machine~\cite{DBLP:conf/icml/DaneshKFK21}. 
%The PIRL framework allows to generate interpretable policies, using the NDPS method~\cite{DBLP:conf/icml/VermaMSKC18}. 
%Coppens et al. train a Soft Decision Tree to mimic a Neural Network policy~\cite{coppens2019distilling}. The explanation of the distribution of actions from a state then takes the form of a succession of heatmaps. 
%The VIPER algorithm produces a surrogate model in the form of a decision tree, making it easier to verify properties~\cite{DBLP:conf/nips/BastaniPS18}.}

% Counterfactual trajectories
One approach consists in generating counterfactual trajectories (state-action sequences) and comparing them with the agent's trajectory. 
In \cite{DBLP:journals/corr/abs-2210-04723}, reward influence predictors are learnt to compare trajectories. The counterfactual one is generated  based on the user's suggestion.
In~\cite{DBLP:journals/corr/abs-1807-08706}, a contrastive policy based on the user's question is built to produce the counterfactual trajectory. 
In the MDP context, Tsirtsis et al. generate optimal counterfactual trajectories that differ at most by $k$ actions~\cite{DBLP:conf/nips/TsirtsisDR21}. The importance score used in this paper is in line with the counterfactual view by evaluating scenarios where a different action took place at a given time.

% EDGE focus
EDGE~\cite{DBLP:conf/nips/GuoWKX21} is a self-explainable model. Like HXP, it identifies the important elements of a sequence. However, EDGE is limited to importance based on the final reward achieved, whereas HXPs allow the study of various predicates. In addition, HXP relies on the transition function (which is assumed to be known) and the agent's policy to explain, whereas EDGE~\cite{DBLP:conf/nips/GuoWKX21} requires the learning of a predictive model of the final-episode reward.
    
%%%%%%%%%%%%%%%%%%%%%%%%%%%%%%%%%%%%%%%%%%%%%%%%%%%%%%%%%%%%%%%%%%%%%%%%
\section{Conclusion}
%%%%%%%%%%%%%%%%%%%%%%%%%%%%%%%%%%%%%%%%%%%%%%%%%%%%%%%%%%%%%%%%%%%%%%%%

% Reminder contributions + experiment results
%Our paper is a follow-up to the work carried out in~\cite{saulieres:hal-04170188}. 
HXP (History eXplanation via Predicates) ~\cite{saulieres:hal-04170188} is a paradigm that, for a given history, answers the question: \emph{``Which actions were important to ensure that the predicate $d$ was achieved, given the agent’s policy $\pi$?"}. To do this, an importance score is computed for each action in the history. 
Unfortunately this calculation is \#W[1]-hard with respect to %the parameter $k$, 
the length of the history to explain. 
To provide explanations for long histories, without resorting to importance score approximation, we propose %devised an approach named 
the Backward-HXP approach: starting from the end of the history, %this involves
it iteratively studies a subsequence, highlighting the most important action in it and defining a new intermediate predicate to study for the next sub-sequence
(working backwards). The intermediate predicate is a \emph{locally-minimal PAXp} corresponding to the state where the most important action took place. 

% Horizon + Generic predicate problem 
In the experiments, we observed that the genericity of a predicate $d$ and the search horizon $l$ influence the importance scores. The more generic the predicate $d$, the greater the probability of finding states at an
horizon $l$ that respect $d$. 
%Thus, the importance scores generated are significant, regardless of $l$. 
Conversely, a less generic predicate makes it more difficult to evaluate an action. 
%In this case, the value of $l$ is discriminating. 
A too specific  %non generic 
predicate $d$ can lead to insignificant importance scores (close to 0) for the respect of $d$. %, as the utility of actions is close to $0$.
In several histories, notably C4 ones, the predicates generated are not generic enough, leading to less interesting explanations.

% Lose track of initial predicate
Although in the examples the actions identified as important are often related to the respect of the initial predicate, this is not always the case.  
If we consider the first redefined predicate as a possible cause of the predicate being satisfied in the final state, then the second redefined predicate is a possible cause of a possible cause.  
The notion of causality can quickly become highly diluted (due to the fact that for computational reasons, we study a single cause at each step).
The user must be aware of this effect when computing B-HXPs.

% Advantages / Drawbacks
%HXP and B-HXP offer the user great diversity in the study of agent behavior through its notion of predicate. 
B-HXP offers the user a credible alternative to forward HXP when studying the behaviour of an agent in long histories.
Furthermore, this approach is agnostic with regard to the agent learning algorithm. 
As done in %noted 
%in
\cite{saulieres:hal-04170188} for HXP, we also made the strong assumption for the use of
%of HXP and 
B-HXP that the transition function is known. It must be known during the explanation phase (not necessarily during training), or at least approximated, for example using an RL model-based method. 

% Future works
More experiments are needed to ensure the quality and scalability of B-HXP, specially in environments with a large number of transitions.
When calculating a \emph{locally-minimal} PAXp, the order in which features are processed is important. An avenue of future work would be to direct the generation of \emph{locally-minimal} PAXp using a feature ordering heuristic, such as LIME~\cite{DBLP:conf/kdd/Ribeiro0G16}. In this way, it would be possible to compare the intermediate predicates and check whether this changes the important actions returned. 

% Conclusion
Our experiments have shown the feasibility of finding the important actions in a long sequence of actions by redefining predicates, working backwards from the end of the sequence. However, we found that the intermediate predicates can quickly become very specific leading to the difficulty of calculating the importance scores of actions w.r.t. these very specific predicates. Further research is required to investigate this point.
The various complexity results obtained %would 
indicate the inherent difficulty of the problem.

%%%%%%%%%%%%%%%%%%%%%%%%%%%%%%%%%%%%%%%%%%%%%%%%%%%%%%%%%%%%%%%%%%%%%%%%
%%% Use this environment to include acknowledgements (optional).
%%% This will be omitted in doubleblind mode.

\begin{ack}
This work has been supported in part by the ForML ANR project ANR-23-CE25-0009.
\end{ack}

%%%%%%%%%%%%%%%%%%%%%%%%%%%%%%%%%%%%%%%%%%%%%%%%%%%%%%%%%%%%%%%%%%%%%%%%

%%%%%%%%%%%%%%%%%%%%%%%%%%%%%%%%%%%%%%%%%%%%%%%%%%%%%%%%%%%%%%%%%%%%%%%%
%%% Use this command to include your bibliography file.
\bibliography{mybibfile}

\end{document}